\documentclass[oneside,letterpaper]{article}

\usepackage{color}
\definecolor{mydarkblue}{rgb}{0,0.03,0.6}
\definecolor{mydarkgreen}{rgb}{0,0.6,0.3}
\usepackage[colorlinks=true, plainpages=false, pdfpagelabels=true, pageanchor=false,hypertexnames=false,linkcolor=mydarkblue,citecolor=mydarkgreen,
pdfauthor={}]{hyperref}

\usepackage{url}
\usepackage{amsmath}
\usepackage{amsthm}
\usepackage{amsfonts}
\usepackage[utf8]{inputenc}
\usepackage{array}
\usepackage[ruled,vlined,linesnumbered]{algorithm2e}
\usepackage{graphicx}
\usepackage{booktabs} % for toprule, midrule, bottomrule in tables
\usepackage{geometry}
\usepackage{authblk}

\newcommand{\R}{\mathbb{R}}
\newcommand{\E}{\mathbb{E}}
\newcommand{\LS}{\text{LS}}   % least-squares
\newcommand{\CLS}{\text{CLS}} % compressed LS
\newcommand{\PCLS}{\text{PCLS}}% partial-compressed LS
\newcommand{\RPC}{\text{RPC}} % robust partial compressed
\newcommand{\order}{\mathcal{O}}
\DeclareMathOperator*{\argmin}{argmin}

\newcommand{\tp}{^\mathsf{T}}

\newcommand{\defeq}{\stackrel{\text{\tiny def}}{=}}  %   another alternative
\newtheorem{defn}{Definition}
\newtheorem{thm}{Theorem}
\newtheorem{lem}{Lemma}
\newcommand{\eye}{I}  % for consistency
\newcommand{\opt}{^{\star}}

\title{Robust Partially-Compressed Least-Squares}

\author[*]{Stephen Becker}
\author[+]{Ban Kawas}
\author[+]{Marek Petrik}
\author[+]{Karthikeyan N. Ramamurthy}

\affil[*]{University of Colorado, Boulder}
\affil[+]{IBM T.J. Watson Research Center}

\begin{document}

\maketitle

\begin{abstract}%
Randomized matrix compression techniques, such as the Johnson-Lindenstrauss transform, have emerged as an effective and practical way for solving large-scale problems efficiently. With a focus on computational efficiency, however, forsaking solutions quality and accuracy becomes the trade-off. In this paper, we investigate compressed least-squares problems and propose new models and algorithms that address the issue of error and noise introduced by compression. While maintaining computational efficiency, our models provide robust solutions that are more accurate---relative to solutions of uncompressed least-squares---than those of classical compressed variants. We introduce tools from robust optimization together with a form of partial compression to improve the error-time trade-offs of compressed least-squares solvers. We develop an efficient solution algorithm for our \emph{Robust Partially-Compressed} (RPC) model based on a reduction to a one-dimensional search. We also derive the first approximation error bounds for \emph{Partially-Compressed} least-squares solutions. Empirical results comparing numerous alternatives suggest that robust and partially compressed solutions are effectively insulated against aggressive randomized transforms.
\end{abstract}

\section{Introduction}

Random projection is a simple and effective dimensionality reduction technique that enables significant speedups in solving large-scale machine learning problems~\cite{dasgupta2000experiments,mahoney2011randomized}. It has been successfully used, for example, in classification \cite{Pilanci-Wainwright, zhang2012recovering}, clustering \cite{boutsidis2010random,fern2003random,urruty2007clustering}, and least-squares problems \cite{drineas2011faster,Pilanci-Wainwright}.  The focus of this paper will be on the latter. We consider the following canonical least-squares estimator, with $A\in\R^{M\times N}$: 
\begin{equation}\label{eq:x-LS} 
x_\LS \defeq \argmin_x \; \frac{1}{2} \, \|Ax-b\|^2 = (A\tp A)^{-1} \,  A\tp \, b
\end{equation}
where $\|\cdot\|$, for vectors, denotes the Euclidean norm throughout the paper, and $A$ has the full column rank. We assume that $M \gg N$ and refer to $x_\LS$ as the solution to the \emph{uncompressed} problem.

When $M$ is very large, solving the least-squares problem in \eqref{eq:x-LS} can be time-consuming and computationally expensive. To gain the necessary speedups, random projections are used. The standard approach to doing so proceeds as follows~\cite{drineas2011faster}. First, we construct a compression matrix $\Phi \in \R^{m \times M}$ from a random distribution such that $\E \left[ \Phi\tp\Phi \right] = \eye$ and $m \ll M$. Then, we solve the \emph{fully compressed} problem:
\begin{equation}\label{eq:x-CLS}
x_\CLS \defeq \argmin_x  \; \frac{1}{2} \, \|\Phi \, (A x-b)\|^2
\end{equation}
Numerous efficient methods for constructing the compression matrix $\Phi$ have been developed; surveys are provided in \cite{boutsidis2013near,drineas2011faster, mahoney2011randomized}. We describe and use several common methods for constructing $\Phi$ in Section~\ref{sec:experiments}.

%In our empirical study in \cref{sec:experiments}, we consider Gaussian \cite{maillard2012linear}, Fast Walsh-Hadamard Transform \cite{kuklinski1983fast}, and Counting sketches \cite{kuklinski1983fast,clarkson2013low}.  Walsh-Hadamard sketches are a particular type of randomized orthonormal systems (ROS), or ``fast Johnson-Lindenstrauss'' transformations, and we refer the reader to, e.g., \cite{mahoney2011randomized,Pilanci-Wainwright} for details.

When $m$ in $\Phi$ is small, the fully compressed least-squares problem in \eqref{eq:x-CLS} is much easier to solve than the uncompressed least-squares in \eqref{eq:x-LS}. However, as demonstrated in our numerical results, compression can introduce significant errors to the solution $x_\CLS$ when compared to the uncompressed solution, $x_\LS$ (see Fig.~\ref{fig:NHIS}, explained in details in Section~\ref{sec:experiments}) and one is forced to consider the trade-off between accuracy and efficiency. As our main contribution, we propose and analyze \emph{two} new models that address this issue and provide a desirable trade-off; enabling robust solutions while preserving small computational complexity. 

\begin{figure}
\centering
\includegraphics[width=0.65\linewidth]{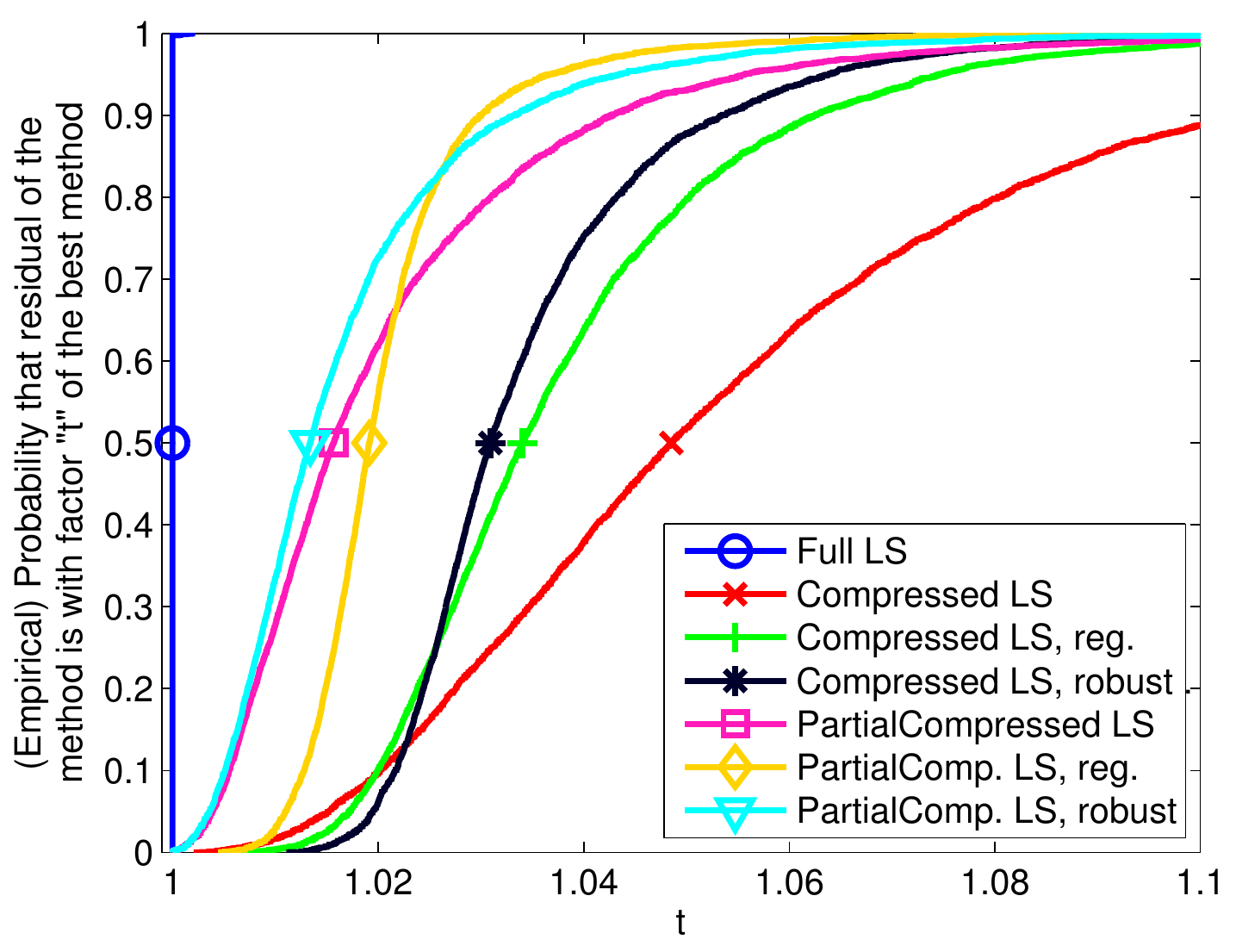}
\caption{Testing the data residual of various compressed least-squares methods.} \label{fig:NHIS}
\end{figure}

Demonstrated in Fig.~\ref{fig:NHIS}---which is based on a real dataset from the National Health Interview Survey from 1992---by simply applying a form of regularization or by applying robust optimization techniques to explicitly model errors introduced by compression, we already observe increased accuracy of solutions compared to the classical compressed solution, $x_\CLS$. Inspired by these results, our main model alleviates the effects of compression on solutions' quality---while maintaining the computational benefits of compressed least-squares---by applying the compression on \emph{only} the computationally intensive terms of ordinary least-squares. Our new model is the following \emph{partially-compressed} least-squares estimator:
\begin{equation}\label{eq:PCLS}
\min_x \frac{1}{2} \, \|\Phi\,A \, x\|^2 - b\tp A \,  x
\end{equation}
and its solution is given by:
\begin{equation} \label{eq:x-PCLS} 
  x_\PCLS \defeq (A\tp\Phi\tp\Phi A)^{-1}A\tp b~.
\end{equation}
Again, notice that only the computationally expensive parts of the ordinary least-squares estimator, which involve inverting $A\tp A$, are compressed.  Also notice, in comparison, that the objective function of the fully compressed least-squares estimator is $\frac{1}{2}\, \|\Phi\,A \, x\|^2 - b\tp \, \Phi\tp\Phi \, A \, x$. 

While not the focus of this paper---since our goal here is to introduce our new estimator in \eqref{eq:PCLS}---it is important to note that the prediction error $\|Ax_\PCLS - Ax_\LS\|$ of the partially compressed solution, $x_\PCLS$, is not always smaller than the error of the fully compressed one, $x_\CLS$. Section~\ref{sec:errors} describes when and why this is the case. Currently, we are investigating this topic in more details and have derived a new model that combines both solutions and outperforms both individually.

Observing that our new estimator still introduces prediction errors when compared to the original uncompressed solution $x_\LS$, we have derived our second model, the \emph{robust partially-compressed} least-squares estimator~(RPC). Described in Section~\ref{sec:partial_definition}, RPC explicitly  models errors introduced by compression and is closely related to robust least-squares regression~\cite{elGhaoui-lebret-1997}.  Leveraging robust optimization techniques makes it possible to reduce the solution error without excessively increasing computational complexity and is a data-driven approach that has been widely used in the last two decades~\cite{ben2009robust}. In our numerical results, we have observed a similar effect to that when applying robust optimization to the fully compressed least-squares solution; increased accuracy and reduction in error. 

While RPC can be formulated as a second-order conic program (SOCP), generic off-the-shelf SOCP solvers may be slow for large problems. Therefore, as one of our contributions, we have developed a fast algorithm based on a \emph{one-dimensional search} that can be faster than CPLEX by over an order of magnitude. Using this fast algorithm, described in Section~\ref{sec:efficient}, the RPC model is asymptotically just as efficient as the non-robust model. Table~\ref{tab:context_results} puts our results in context of prior related work and all model variants we consider are presented in Table~\ref{tb:methods}.

\begin{table*}
\centering
\begin{tabular}{l|lll}
\toprule
 & Least Squares & Ridge Regression & Robust Least Squares \\ 
\midrule
No Compression & many & e.g., \cite{Boyd2004} & \cite{elGhaoui-lebret-1997} \\ 
Partial Compression & \textbf{new}: \eqref{eq:PCLS} & \textbf{new}: \eqref{eq:partial_ridge} & \textbf{new}: \eqref{eq:RPC} \\ 
Full Compression & e.g., \cite{drineas2011faster} & e.g., \cite{Boyd2004} & \textbf{new} (but algo.\ via El Ghaoui) \\ 
\bottomrule
\end{tabular} 
\caption{Our work in context of previous results. The equation numbers point to the objective functions for each method.} \label{tab:context_results}
\end{table*}

Our empirical results, discussed in Section~\ref{sec:experiments}, show that both \emph{partially-compressed} and \emph{robust partially-compressed} solutions can outperform models that use full compression in terms of quality of solutions. We also show that compressed variants are more computationally efficient than ordinary least-squares, especially as dimensions grow.

%Random projection has been used in several lines of works; surveys are provided in \cite{drineas2011faster, mahoney2011randomized, boutsidis2013near}.  We do not elaborate further on ROS sketches since they are well described in the above surveys.  The most recent results are discussed in \cite{Pilanci-Wainwright}. 

\newcommand{\citeasnoun}[1]{\citeauthor{#1} \citeyear{#1}}

\newcommand{\subjectto}{\mbox{s.t.} &}
\newcommand{\stc}{\\[1ex] \subjectto}
\newcommand{\minimize}[1]{\min_{#1} &}
\newcommand{\maximize}[1]{\max_{#1} &}
% The following is the style with an extra alignment for equalities, as used % % in my thesis:
%\newenvironment{mathprog}{\begin{equation}\begin{array}{>{\displaystyle}l>{\displaystyle}r>{\displaystyle}l>{\displaystyle}l}}{\end{array}\end{equation}}
% The following is the style where there is NO extra alignment for equalities, but there % is extra column for quantifications, such as for all i...
\newenvironment{mprog}{\begin{equation}\begin{array}{>{\displaystyle}l>{\displaystyle}l>{\displaystyle}l}}{\end{array}\end{equation}}
\newenvironment{mprog*}{\begin{equation*}\begin{array}{>{\displaystyle}l>{\displaystyle}l>{\displaystyle}l}}{\end{array}
\end{equation*}}
\newcommand{\cs}{\\[1ex] & }
\renewcommand{\t}{^{T}}

\section{Robust Partially-Compressed Least-Squares} \label{sec:partial_definition}

In this section, we describe how to incorporate robustness in our \emph{Partially-Compressed} least-squares model~\eqref{eq:PCLS}. As described above, our objective is to enhance solutions' quality and increase robustness against noise and errors introduced by compression. One way of improving robustness is to use ridge regression, which when applied to our model~\eqref{eq:PCLS}, we obtain the following formulation:
\begin{equation} \label{eq:partial_ridge}
\min_x \frac{1}{2} \, \|\Phi\,A \, x\|^2 - b\tp A \, x + \mu \| x \|^2,
\end{equation}
for some regularization parameter $\mu$. One caveat of using ridge regression is that it does not capture the error structure introduced by compression, which differs significantly from that present in the data of the original uncompressed ordinary least-squares problem. Robust optimization~\cite{ben2009robust}, however, enables us to do exactly that and allows us to explicity model the error structure. The following is our \emph{Robust Partially-Compressed}~(RPC) estimator:
\begin{equation} \label{eq:RPC}
    x_\RPC = \argmin_x \; \max_{\|\Delta P\|_F \le \rho}\; \frac{1}{2}\|(P+\Delta P)x\|^2 - b\tp Ax
\end{equation}
where $P=\Phi A$ and $\Delta P$ is a matrix variable of size $m\times N$. The general formulation of the problem allows for a more targeted model of the noise that captures the fact that $\|\Phi A x\|$ is a random variable while $b\tp A x$ is not. That is, the uncertainty is restricted to the data matrix $P$ alone since the partial compression does not introduce any noise in the right-hand side.

Without compression, it is worth noting that applying robust optimization techniques to the \emph{ordinary} least-squares problem yields the same solution as applying ridge regression with a data-dependent parameter~\cite{elGhaoui-lebret-1997}. As we will show, this is not the case in our setting, as robust partially-compressed least-squares does not reduce to ridge regression. Empirically, we have also seen that robust partially-compressed least-squares is more likely to yield better results than ridge regression and has more intuition built behind it. 

All of the above, motivated us to focus more on our RPC~\eqref{eq:RPC} model and to derive a corresponding efficient solution algorithm. In what follows, we show that RPC~\eqref{eq:RPC} can be formulated as a second-order conic program~(SOCP) that can be solved via off-the-shelf solvers such as CPLEX, and in Section~\ref{sec:efficient}, we propose an alternative way to solving RPC and derive an efficient solution algorithm based on a reduction to one-dimensional search.
  
\subsection{SOCP Formulation of RPC}
While the inner optimization in \eqref{eq:RPC} is a non-convex optimization problem, we show in the following lemma that there exists a closed-form solution.
\begin{lem} \label{lem:inner_optimization}
The inner maximization in \eqref{eq:RPC} can be reformulated for any $x$ as: 
\begin{equation} \label{eq:wc-resid}
\max_{\|\Delta P\|\le \rho} \; \|(P+\Delta P)x\|^2  = \left( \|Px\|+\rho\|x\|\right)^2~.
\end{equation}  
In addition, the maximal value is achieved for $\Delta P = \frac{\rho}{\|Px\|\|x\|}P x x\tp$.
\end{lem}
\begin{proof}
The objective function can be upper-bounded using the triangle inequality: 
\begin{align}
\nonumber
 \max_{\|\Delta P\|\le \rho}\|(P+\Delta P)x\|^2  &\le \max_{\|\Delta P\|\le \rho} \left( \|Px\|+\|\Delta Px\|\right)^2 \\
 \nonumber
        &\le \left( \|Px\|+\rho\|x\|\right)^2 ~.
\end{align}
To show that this bound is tight, consider $\overline{\Delta P} = \frac{\rho}{\|Px\|\|x\|}P x x\tp$. It can be readily seen that $\|\overline{\Delta P}\|_F = \rho$.  Then by algebraic manipulation:
\begin{align}
\nonumber
    \max_{\|\Delta P\|\le \rho}\|(P+\Delta P)x\|^2  &\ge \|(P+ \overline{\Delta P}) \, x\|^2  \\
 \nonumber
	&=  \left( \|Px\|+\rho\|x\|\right) ^2~.
\end{align}
\end{proof}

Using Lemma~\ref{lem:inner_optimization}, the robust partially-compressed estimator $x_\RPC$ is the optimal solution to:
\begin{equation} \label{eq:rpc_nice}
\min_x \; \frac{1}{2} \left( \|Px\| + \rho \|x\| \right)^2 - b\tp A x~.
\end{equation}  

We now analyze the structure of the optimal solution and point to connections and differences in comparison to results from ridge regression.
\begin{thm} \label{thm:structure}
The optimal solution $x_\RPC$ to \eqref{eq:rpc_nice} must satisfy:
\begin{equation} \label{eq:xRPC}
    x_\RPC = \frac{1}{\alpha+ \rho \, \beta} (\alpha^{-1} \, P\tp \, P + \rho \, \beta^{-1} \, \eye )^{-1} A\tp \, b~,
\end{equation}
such that $\alpha=\|P x_\RPC\|$ and $\beta=\|x_\RPC\|$, or $x_\RPC=0$ if $A\tp b=0$.
\end{thm}
\begin{proof}
The theorem follows the first-order optimality conditions. The function \eqref{eq:rpc_nice} is everywhere convex, and differentiable everywhere except at $x = 0$. We can show the solution $x=0$ is only optimal if $A\tp b = 0$. The objective at $x=0$ is $0$. If $A\tp b \neq 0$, then for sufficiently small $t>0$, the point $t A\tp b$ gives a strictly negative objective (since $t^2=o(t)$ as $t\rightarrow 0$), hence $x=0$ is not optimal. If $x\neq 0$, the following first-order conditions are necessary and sufficient:
\[
    0 = ( \|Px\|+\rho\, \|x\|)\left( \frac{P\tp P x}{\|Px\|} + \rho\frac{x}{\|x\|} \right) - A\tp b,
\]
from where we derive \eqref{eq:xRPC}.  The theorem follows directly from setting $\alpha$ and $\beta$ to the required values.
\end{proof}

Theorem~\ref{thm:structure} shows that the optimal solution to the robust partially-compressed least-squares problem is structurally similar to a ridge regression solution. The two main differences are that there are two parameters, $\alpha$ and $\beta$, and these parameters are data-dependent. When setting $\rho$ to 1---which is what we have done in our empirical study, one advantage over ridge regression would be that there is no need to fine-tune the regularization parameter, $\mu$, and one can rely on only data-driven parameters $\alpha$ and $\beta$. Even when there is a need to fine-tune the free parameter $\rho$ in RPC---which we have not done in our results and simply sat $\rho$ to be equal to 1---$\rho$ has a structural meaning associated with it; $\rho$ is the size of the uncertainty set in~\eqref{eq:RPC} and~\eqref{eq:wc-resid} and one can quickly build an intuition behind how to set its value, which is not the case for the regularization parameter $\mu$. In a current investigation, which is out of the scope of this paper, we are building connections between $\rho$ and the compression dimension $m$, which will enable us to appropriately set $\rho$ as a function of $m$.

Note that Theorem~\ref{thm:structure} does not provide a method to calculate $x_\RPC$, since $\alpha$ and $\beta$ depend on $x_\RPC$. However, given that \eqref{eq:rpc_nice} is a convex optimization problem, we are able to reformulate it as the following second-order conic program~(SOCP) in standard form:
\begin{mprog} 
\label{eq:rpc_socp}
\minimize{x,t,u,z} \frac{1}{2} z - b\t A x
\stc \| P x \| \le t~, \quad \rho \, \| x \| \le u~, \quad
 \left\| \begin{matrix} t + u \\ z - \frac{1}{4} \end{matrix} \right\| \le z + \frac{1}{4}~.
\end{mprog}
The last constraint in this program translates to $z \ge (t+u)^2$.

While efficient polynomial-time algorithms exists for solving SOCP problems, they are typically significantly slower than solving least-squares. Therefore, to achieve practical speedup, we need to derive a more efficient algorithm. In fact we propose a reduction to a one-dimensional optimization problem in Section~\ref{sec:efficient}.

\section{Efficient Computation of RPC} \label{sec:efficient}

\newcommand{\dual}{\tau}		
\newcommand{\dualPrime}{\tau'}
\newcommand{\lag}{\mathcal{L}}
\newcommand{\h}{h_{\dual}}

As Section~\ref{sec:partial_definition} shows, the RPC optimization \eqref{eq:rpc_nice} represents a convex second-order conic problem. However, simply using convex or SOCP solvers is too slow when applied to large problems. In this section, we describe a faster approach based on a reduction to a one-dimensional search problem.

\IncMargin{1em}
\begin{algorithm}
\begin{small}
\KwIn{$A$, $b$, $\Phi$, $P = \Phi A$, $\rho$}
\KwOut{$x$}
$U \, \Sigma \, V\tp \leftarrow \operatorname{SVD}(P)$\;
%\tcp{Initialize the dual variable to $1/2$ maximum}
$\tau \leftarrow \rho \, \|b\|_2 / 2$  \tcp*{Initialization}
\tcp{Solve $x \leftarrow \argmin_x h_{\tau_k}(x)$ }
\While{$\left|\, \|\Sigma y \| \gamma_k - 1 \right| \le \epsilon$}{
%\tcp{Newton method, $\phi$ is monotone}
$\gamma_k \leftarrow  \arg\min_\gamma \phi(\gamma) = \sum_{i=1}^N  \frac{\bar{b}_i^2}{ \left(\gamma \sigma_i^2 + \rho \right)^2} - 1$ 
$y_k \leftarrow \frac{1}{\tau} V\tp (P\tp P + \gamma_k \eye)^{-1} A\tp b $ \;
\tcp{When $\tau=\tau^\star$ then $\alpha = \|\Sigma y \|$} 
 \label{ln:mystery_update} 
$\tau_{k+1} \leftarrow \tau_k \|\Sigma y_k \| \, \gamma_k $ \;
}

\tcp{Recover the solution}
$\alpha \leftarrow \frac{\tau}{1 + \rho \gamma^\star}$ \tcp*{Using: $\alpha + \rho \, \beta = \tau$}
$\beta \leftarrow \frac{\tau - \alpha}{\rho}$ \;
$x \leftarrow \frac{1}{\beta} V y$ \;
\end{small}
\caption{Efficient Algorithm for Solving RPC} \label{alg:superfast}
\end{algorithm}

First, re-formulate the optimization problem \eqref{eq:rpc_nice}  as:
\begin{equation}
 \min_{x,t} ~ \frac{1}{2} t^2 - b\tp A x \quad \text{s.t.} \quad
 \|Px\|+\rho \|x\| \le t
    \label{eq:optimize_t}
\end{equation}
Our goal is to derive and then solve the dual problem. 
The Lagrangian of \eqref{eq:optimize_t} is
\[ \lag(x,t,\dual) =  \frac{1}{2}t^2 - b\tp A x + \dual \left( \|Px\| + \rho \|x\| - t \right) \]
Since strong duality conditions hold, we solve the one-dimensional dual maximization problem $\max_{\tau\ge 0} g(\tau)$ where $g(\tau)$ is given as
\begin{align}
    &\min_t \left(  \frac{1}{2}t^2 - \dual t \right) + \min_x \, \dual\left(\|Px\|+\rho\|x\|\right) - b\tp A x \notag \\
             &= -\frac{1}{2}\dual^2 + \min_x \, \underbrace{\dual\left(\|Px\|+\rho\|x\|\right) - b\tp Ax}_{\h(x)}~. \label{eq:ht}
\end{align}
The second equality follows since $\|Px\|+\rho\|x\| = t = \tau$ for the optimal primal and dual solution. Observe that $h_\tau(x)$ is positive homogeneous in $x$ and therefore:
\begin{equation}
\min_x\; \h(x)= \begin{cases}
        -\infty & \dual < \dual^\star \quad\text{(Case 1)}  \\
        0 & \dual = \dual^\star \quad\text{(Case 2)}\\
        0 & \dual > \dual^\star\quad\text{(Case 3)}
    \end{cases}
    \label{eq:cases}
\end{equation}
where $\dual^\star \ge 0$ is the optimal dual value. 

Intuitively, to solve for the optimal solution, we need to find the maximal value of $\tau$ such that $h_\tau(x) = 0$. Appendix~\ref{appx:algorithm} derives the approach that is summarized in Algorithm~\ref{alg:superfast}. Observe that the function $h_\tau(x)$ is convex. The main idea is to reduce the optimization to a single-dimensional minimization and solve it using Newton method. We also use the SVD decomposition of $P$ to make the search more efficient so that only a single $\order(N^3)$ step is needed.

%While we observe that the update in Line~\ref{ln:mystery_update} of Algorithm~\ref{alg:superfast} empirically leads to a convergence in a small number of iterations, we were not able to show that the convergence is always guaranteed. One can replace the update by a bisection rule which provably converges.

In terms of the computational complexity, Algorithm~\ref{alg:superfast} requires $\order(m N^2 + N^3)$ operations to compute the singular value decomposition and to multiply $P\tp P$. All operations inside of the loop are dominated by $\order(N^3)$. The number of iteration that is needed depends on the desired precision. Table~\ref{tab:complexity} compares the asymptotic computational complexity of the proposed robust partial compression with the complexity of computing the full least-squares solution.

\begin{table*}
\centering
\begin{small}
\begin{tabular}{l|llll}
\toprule
 & \textbf{Least Squares} & \multicolumn{3}{c}{\textbf{Robust Partial Compression}} \\ 
\midrule
Compression &  & \multicolumn{1}{c}{\textsl{Gaussian}} & \multicolumn{1}{c}{\textsl{Walsh-Hadamard}}  &  \multicolumn{1}{c}{\textsl{Counting}} \\
Comp. Time &  & $\order(m \, M \, N)$ & $\order(M \, \log M N)$ & $\order(nnz)$  \\ 
Solution Time & $\order(M \, N^2)$ & $\order(m N^2 + N^3)$ &  $\order(m N^2 + N^3)$  & $\order(m N^2 + N^3)$  \\ 
\textbf{Total Time} & $\order( M \, N^2)$ & $\order(m \, M \, N + m N^2)$ & $\order(M \, \log M N + m N^2)$ & $\order(nnz + m N^2)$ \\
\bottomrule
\end{tabular} 
\end{small}
\caption{Asymptotic computational complexity of various compression methods. Symbol $nnz$ denotes the number of non-zero elements in $A$ we are assuming that $m \gg N$ and $M \gg N$.    
    } \label{tab:complexity}
\end{table*}

\section{Approximation Error Bounds} \label{sec:errors}

Compressing the least-squares problem can significantly speed up the computation, but it is also important to analyze the quality of the solution of the compressed solution. Such analysis is known for the fully compressed least-squares problem (e.g. \cite{Pilanci-Wainwright}) and in this section, we derive bounds for the partially-compressed least-squares regression. 

First,  the following simple analysis elucidates the relative trade-offs in computing full or partial projection solutions. Let $x\opt$ be the solution to the full least-squares problem \eqref{eq:PCLS} and $z\opt = b - A x\opt$ be the residual. Recall that $A\tp z\opt = 0$. Now when $x_\CLS$ is the solution to \eqref{eq:x-CLS}, then:
\begin{align}
\nonumber
 x_\CLS &= ( A\tp \Phi\tp \Phi A)^{-1} A\tp \Phi\tp \Phi b \\
 \nonumber
 &= x\opt + ( A\tp \Phi\tp \Phi A)^{-1} A\tp \Phi\tp \Phi z\opt
\end{align}
 On the other hand, the solution $x_\PCLS$ to \eqref{eq:x-PCLS} satisfies:
\[
 x_\PCLS = ( A\tp \Phi\tp \Phi A)^{-1} A\tp  b = ( A\tp \Phi\tp \Phi A)^{-1} A\tp  A x\opt 
\]
The error in $x_\CLS$ is additive and is a function of the remainder $z\opt$. The error in $x_\PCLS$ is, on the other hand, multiplicative and is independent of $z\opt$. As a result, a small $z\opt$ will favor the standard fully compressed least-squares formulation, and a large $z\opt$ will favor the new partial compressed one.

%We now build upon results in \cite{Pilanci-Wainwright} to prove that the optimal solution of the partial projection problem is close to the true solution of the least-squares problem. 
We will now show that, in the sense of the following definition, the residual of the optimal solution of the partial projection problem is close to the residual true solution of the least-squares problem.
%The following definition defines how we measure the closeness to the solution.
\begin{defn}[$\epsilon$-optimal solution]
We say that a solution $\hat{x}$ is $\epsilon$-optimal if it satisfies
\begin{equation} \label{eq:epsilon-optimal}
\frac{\|A(\hat{x} - x_\LS)\|}{\|A x_\LS\|} \leq \epsilon, \quad \epsilon \in (0,1)
\end{equation}
where $x_\LS$ is an optimal solution of the original high-dimensional system \eqref{eq:x-LS}.
\end{defn}
For sub-Gaussian and ROS sketches, we can show that results in \cite{Pilanci-Wainwright} can be extended to bound approximation errors for partially-compressed least-squares based on the definition of $\epsilon$-optimal above. 
These results are nearly independent of the number of rows $M$ in the data matrix (except for how these affect $\|Ax_\LS\|$). 
The main guarantees for unconstrained least-squares are given in the following theorem that provides an exponential tail bound: 
%TODO Describe pilanci's techniques a bit, mention exponential form of bound, and how m is
% dependent on epsilon, but independent of N.
\begin{thm}[Approximation Guarantee] \label{thm:approx-error-bounds}
Given a normalized sketching matrix $\Phi \in \R^{m \times M}$, and universal constants $c_0,c_0', c_1, c_2$, the sketched solution $x_{PCLS}$ \eqref{eq:x-PCLS} is $\epsilon$-optimal \eqref{eq:epsilon-optimal} with probability at least $1 - c_1\exp(-c_2\,m\,\epsilon^2)$, for any tolerance parameter $\epsilon \in (0,1)$,  \underline{when} the sketch or compression size $m$ is bounded below by
\begin{itemize}
\item[(i)]  $m > c_0 \frac{rank(A)}{\epsilon^2}$, if $\Phi$ is a scaled  sub-Gaussian sketch 
\item[(ii)] $m > c'_0 \frac{rank(A)}{\epsilon^2}\log^4(N)$, if $\Phi$ is a scaled randomized orthogonal systems (ROS) sketch
\end{itemize}
\end{thm}
See Appendix~\ref{app:error_proof} for the proof.

By ``scaled'' sketch, we mean $\E(\Phi\tp\Phi)=I$, since for partial compression, scaling $\Phi$ does affect the answer, unlike full compression. For example, in the Gaussian case, we draw the entries of $\Phi$ from $\mathcal{N}(0,\frac{1}{m})$ instead of $\mathcal{N}(0,1)$.

%\subsection{regularized variants}%possibly add it as a subsection to robustCLS
%\input{guaranteesForPCLS} %moved as a subsection to robustCLS 

\section{Empirical Results} \label{sec:experiments}

\begin{table*}
\centering
\begin{small}
\begin{tabular}{ll}
\toprule
 Least Squares variant &  Objective \\ 
\midrule
Least Squares &  $ \min_x \; \frac{1}{2} \, \|A \, x - b\|^2 $ \\ 
Compressed LS (i.e., \emph{full} compression) &  $ \min_x\; \frac{1}{2} \, \|\Phi\,( A \, x - b ) \|^2 $  \\ 
Regularized Compressed LS &  $ \min_x\; \frac{1}{2} \, \|\Phi\,( A \, x - b ) \|^2  + \frac{\mu}{2} \| x \|^2$  \\ 
Robust Compressed LS &  $ \min_x \, \max_{\|[\Delta P,\Delta B] \|_F \leq \rho}\; \frac{1}{2} \, \|(\Phi A+\Delta P) \, x - (\Phi + \Delta B)b ) \|^2  $  \\ 
Partial Compressed LS & $ \min_x\; \frac{1}{2} \, \|\Phi\,A \, x\|^2 - b^T  A \, x  $ \\ 
Regularized Partial Compressed LS & $ \min_x \;\frac{1}{2} \, \|\Phi\,A \, x\|^2 - b^T A \, x + \frac{\mu}{2} \| x \|^2 $ \\ 
Robust Partial Compressed LS &  $ \min_x\, \max_{\|\Delta P \|_F \leq \rho} \frac{1}{2} \, \|(\Phi\,A+\Delta P) x\|^2 - b^T  A \, x $ \\
BLENDENPIK & $ \min_x \; \frac{1}{2} \, \|A x - b\|^2 $ via preconditioned LSQR \cite{avron2010blendenpik} \\ 
\bottomrule
\end{tabular} 
\end{small}
\caption{Methods used in the empirical comparison} \label{tb:methods}
\end{table*}

Our focus in this section is on the improvement of the solution error in comparison with the non-compressed least squares solution and the improvement over regular full compression. We also investigate the computational speed of the algorithms and show that partial compression is just as fast as full compression (and hence sometimes faster than standard least-squares), and that robust partial compression is only roughly twice as slow (and asymptotically it is the same cost).

Table~\ref{tb:methods} summarizes the methods that we consider for the empirical evaluation. All methods were implemented in MATLAB with some of the random projection code implemented in C and multithreaded using the pthread library; all experiments were run on the same computer. For solving ordinary least squares, we consider two methods: 1) directly use Matlab's least-square solver $A \backslash b$ (which is based on LAPACK), and 2) solving the normal equations: $(A^T A) \backslash b$. The latter method is significantly faster but less numerically stable for ill-conditioned matrices.

For completeness, we compare with (ridge-)regularized and robustified versions of the standard compressed LS problem~\eqref{eq:x-CLS}. The robust version is solved following the algorithm outlined in \cite{elGhaoui-lebret-1997} since this can be treated as a robust  ordinary least squares problem.

We first investigate accuracy using a data set
%Finally, the third data set originates
from the National Health Interview Survey from 1992, containing 44085 rows and only 9 columns; since it is highly overcomplete and contains potentially sensitive data, it is a good candidate for sketching. To test over this, we do 100 realizations of 5000 randomly chosen training data and 10000 testing data, and for each realization draw 50 random Walsh-Hadamard sketches with $m=10N$.
The residual on the testing data (median over all $5000$ realizations) is shown in Fig.~\ref{fig:NHIS}. For robust variants, we set $\mu$ to be 5 times the minimum eigenvalue of $A\tp\Phi\tp\Phi A$, and for robust variants we set $\rho=1$. 

The figure presents the results similar to a CDF or a ``performance profile'' as used in benchmarking software; a smaller area above the curve indicates better performance. A point such as $(0.5,1.02)$ means that on half the simulations, the method achieved a residual within a factor of $1.02$ of the least-square residual. 

\begin{figure}
		\centering
	        \includegraphics[width=0.65\linewidth]{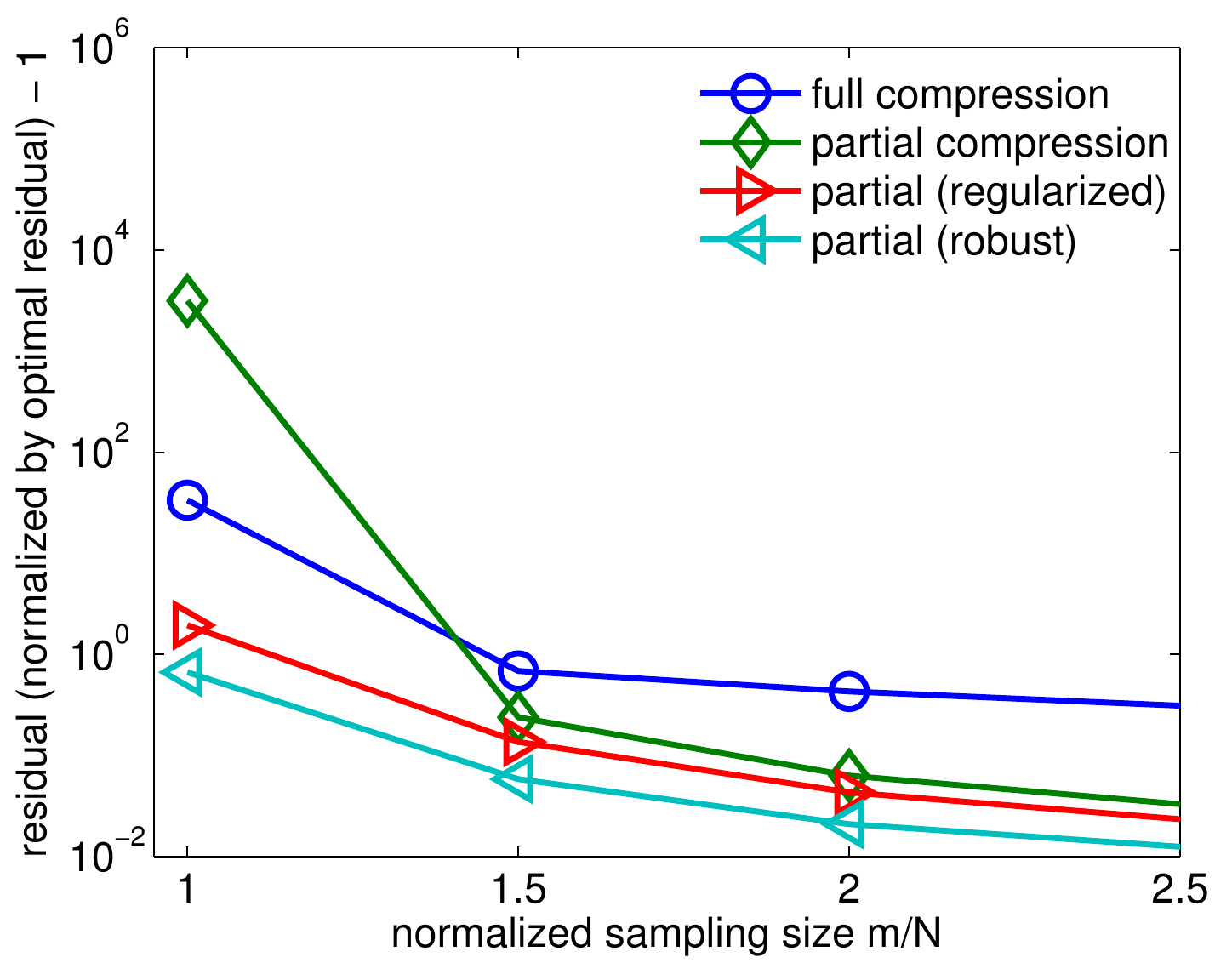}
		\caption{For very high compression ($m/M$ very small), with even $m=N$, robustness/regularization is beneficial.}       \label{fig:smallGamma}     
\end{figure}

There are two clear observations from the figure:  partial compression gives lower residuals than full compression, and the regularized and robust variants may do slightly worse in the lower-left (i.e., more bias) but better in the worst-case upper-right (i.e., less variance).
Put another way, the robust and regularized versions have stronger tail bounds than the standard versions. We also see a slight benefit of robustness over regularization, though the effect depends on how $\mu$ and $\rho$ are chosen.

%To evaluate the effect of adding robustness to partial compression, we consider the data-set originating in the  reinforcement learning benchmark. Figure~\ref{fig:pendulum_performance_profiles} shows the performance profiles of the least squares algorithms. The figure shows the distribution of the accuracy achieved for many compression matrices. The horizontal axis represents the relative accuracy error as described above and the vertical axis represents the fraction of sample (of $\Phi$) that fall within that accuracy. A smaller area above the curve indicates an overall better performance. 

%The results in Figure~\ref{fig:pendulum_performance_profiles} demonstrate that robustness can lead to a significant improvement in the quality of the solution. The data matrix in this case is also very sparse and doing partial compression least squares can result in a very bad accuracy. However, once robustness is added to the model, the solution quality is strictly better than other methods.

\begin{figure*}[ht]
	\centering
		\includegraphics[width=0.4\linewidth]{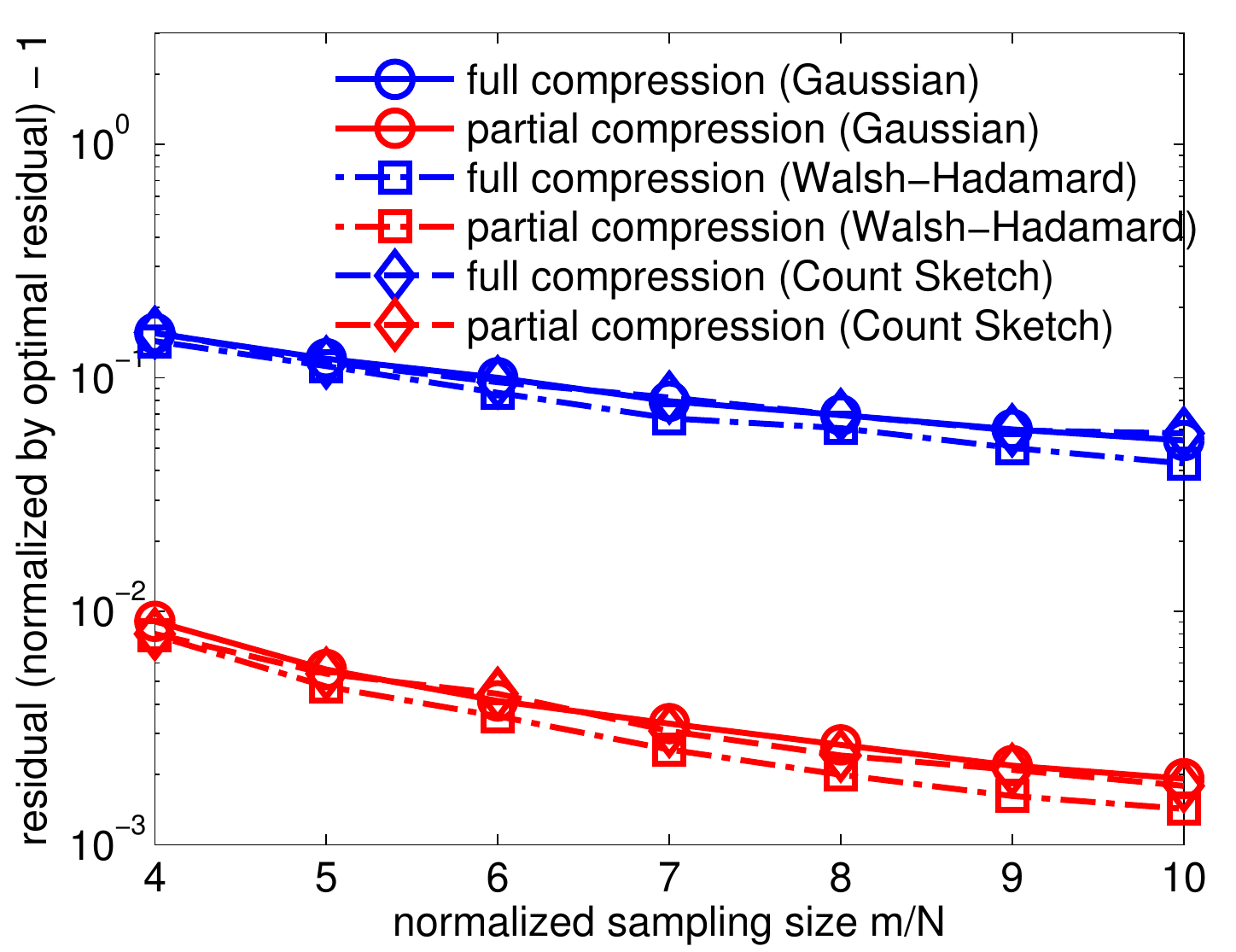}
		\hspace{0.7in} 
		\includegraphics[width=0.4\linewidth]{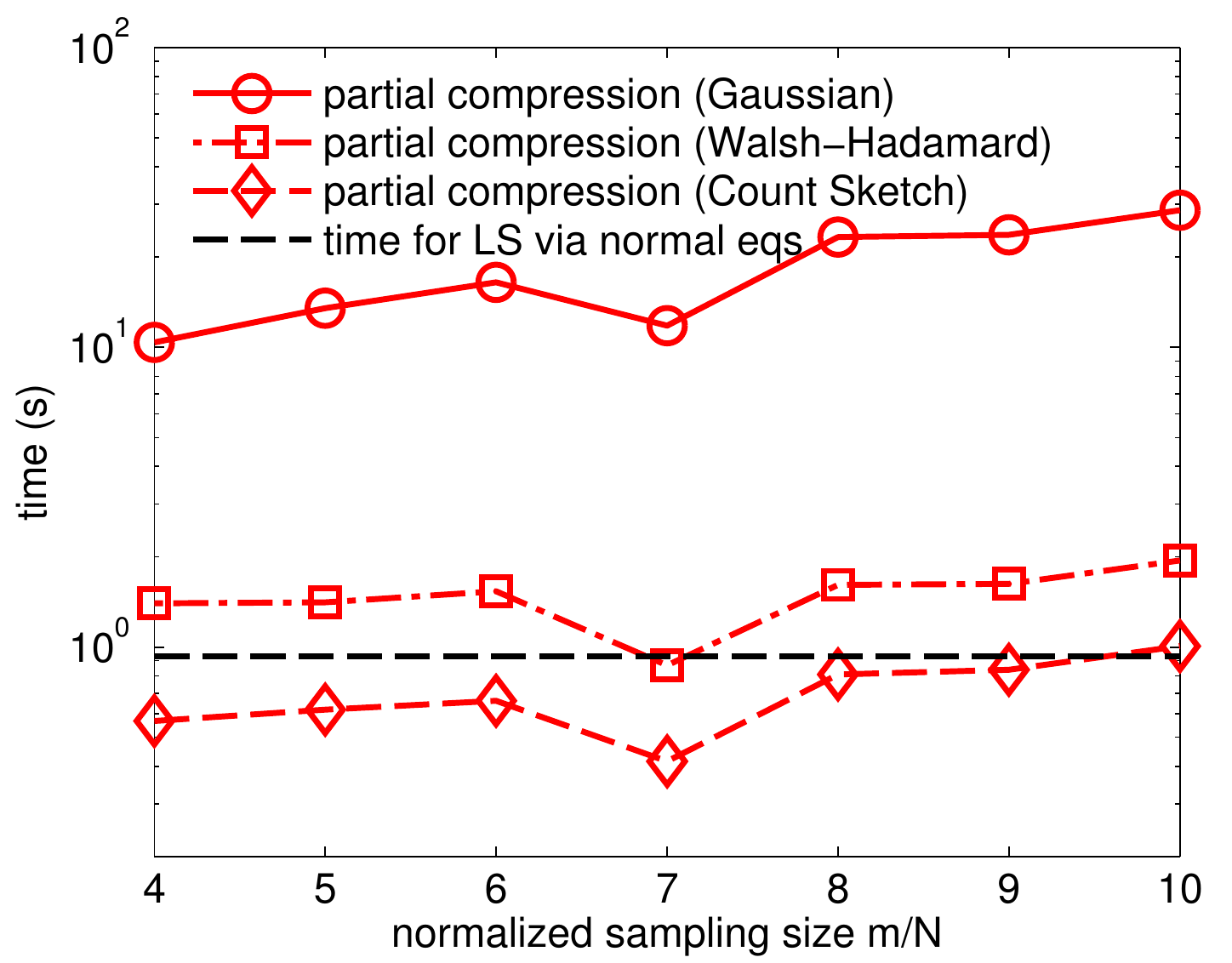} 
		%  To see the details:
		%\includegraphics[width=.45\textwidth]{figs/speedTest3_sweepM_v01_resids_annotated.pdf}
		%\includegraphics[width=.45\textwidth]{figs/speedTest3_sweepM_v01_times_annotated.pdf}
	\caption{
	Accuracy (left) and speed (right) for partial/full compressed least squares for three types of sketching matrices as a function of the amount of compression. Matrix dimensions were $30000\times 750$.
	}
	\label{fig:speedTest3_smallersize}
\end{figure*}

	\begin{figure*}[ht]
	\centering
		\includegraphics[width=0.4\linewidth]{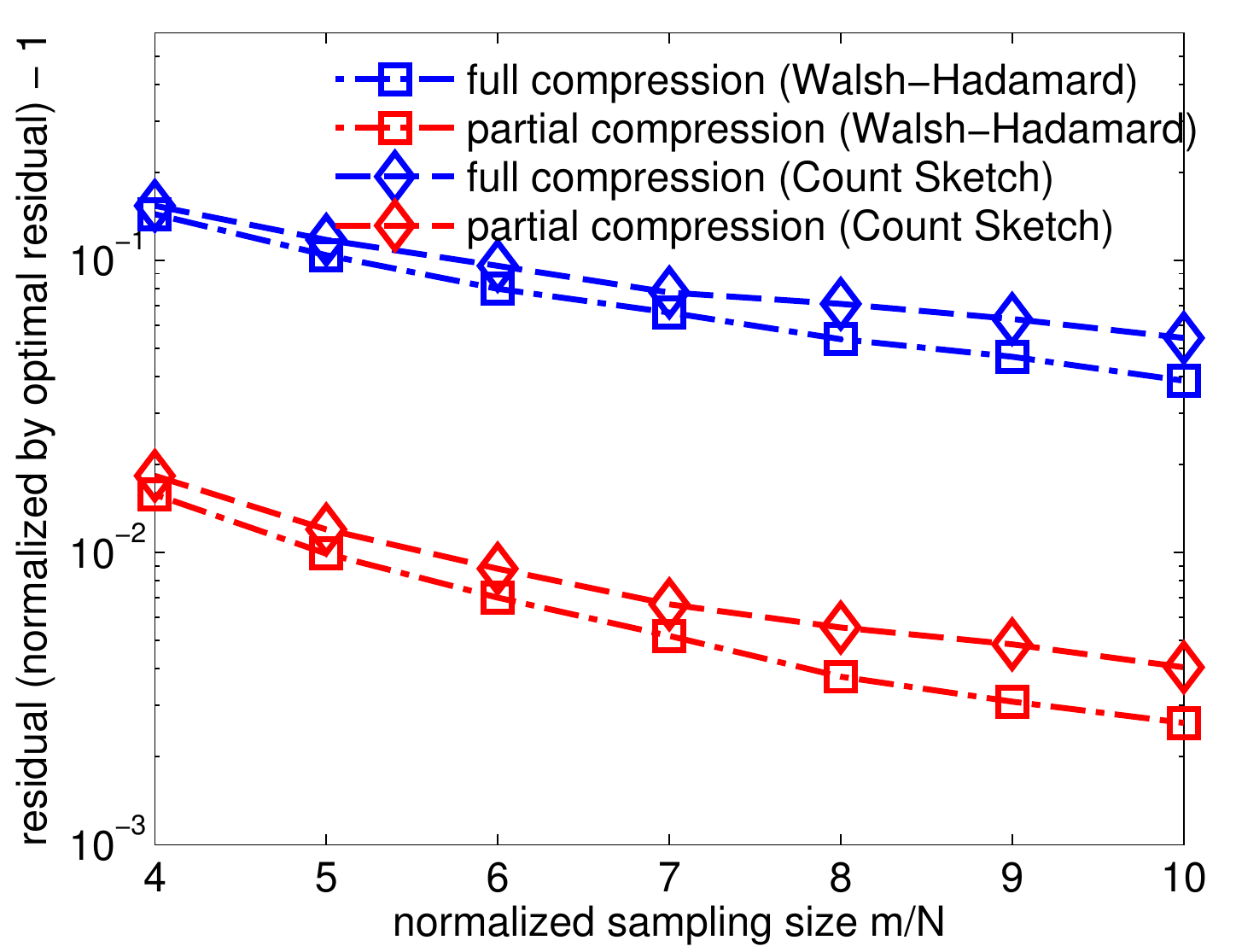}
				\hspace{0.7in} 
		\includegraphics[width=0.4\linewidth]{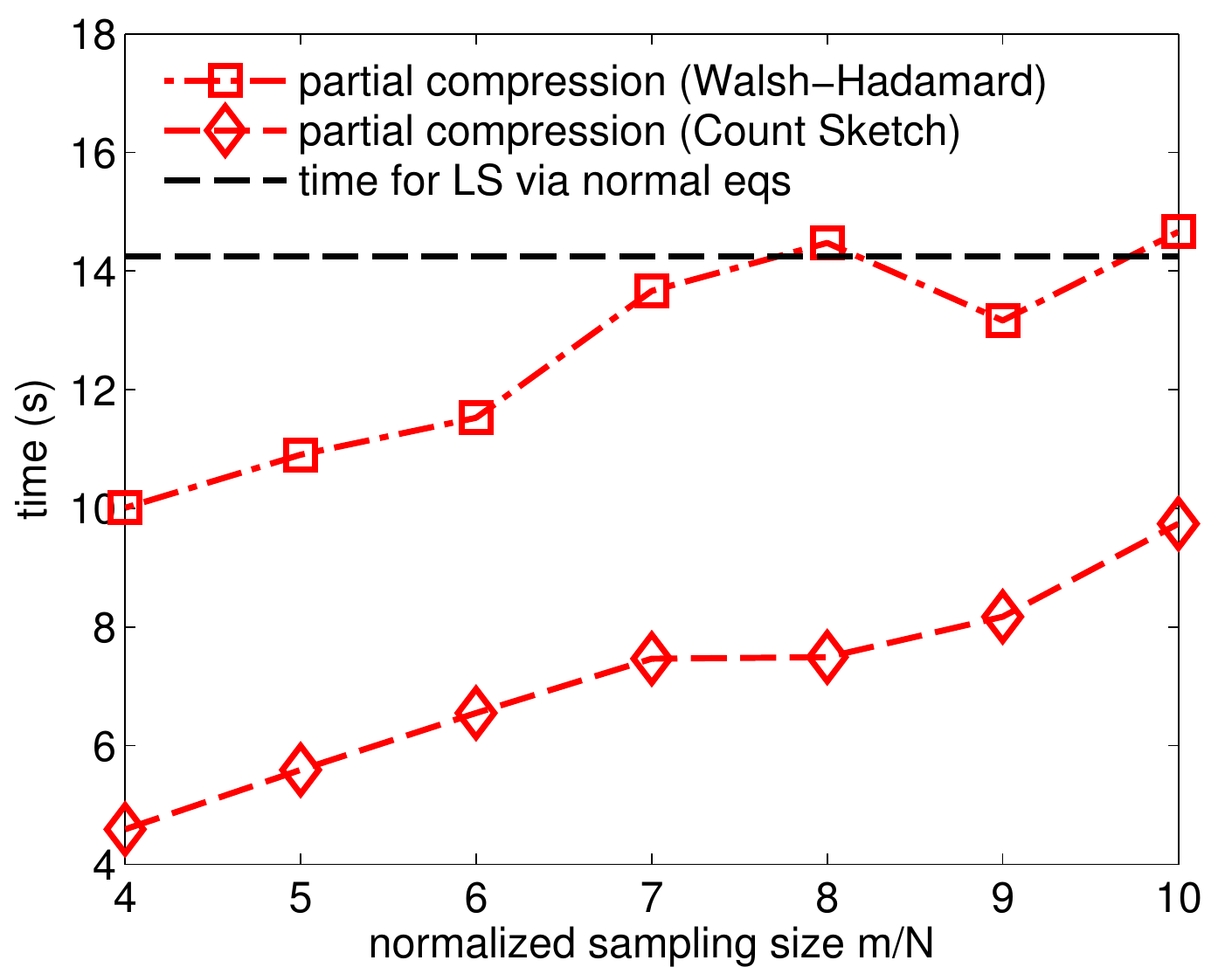}
		%  To see the details:
		%\includegraphics[width=.45\textwidth]{figs/speedTest3_sweepM_v02_resids_annotated.pdf}
		%\includegraphics[width=.45\textwidth]{figs/speedTest3_sweepM_v02_times_annotated.pdf}
	\caption{
	%Speed and accuracy for partial/full compressed least squares for two types of sketching matrices.
	%Left: accuracy.
	%Right: speed results (both partial and full were identical, so only the results of partial compression shown).
	Accuracy (left) and speed (right) for partial/full compressed least squares for two types of sketching matrices as a function of the amount of compression.  Matrix dimensions were $40000\times 2000$. The matrix was too large for Gaussian compression. %Each point is the median over 4 trials.
	}
	\label{fig:speedTest3_largersize}
\end{figure*}
	
The next dataset represents a set of synthetically generated matrices used in \cite{avron2010blendenpik}. The authors consider three distinct types of matrix based on their coherence: 1) incoherent, 2) semi-coherent, and 3) coherent. 
Figure~\ref{fig:speedTest3_smallersize} shows the results on a $30000 \times 750$ \emph{incoherent} matrix for all three types of matrix compression. The data matrix has a condition number $10^4$. To compare multiple methods, we use the same compression matrix to evaluate each methods. The accuracy of a solution $\widehat x$ is defined with respect to the residual of $x_\LS$, namely $\left(\|A\widehat{x}-b\|/\|A x_\LS-b\|\right)-1$.

The results in Figure~\ref{fig:speedTest3_smallersize}  illustrate that the partial compression significantly improves on the residual of the solution in comparison with using full compression. In terms of the timing results, the computation time is dominated by the random projection. As a result, only the fastest projection based on the count sketch significantly outperforms ordinary least squares. Figure~\ref{fig:speedTest3_largersize} shows the same results for a larger matrix with dimensions $40000\times 2000$. Now that the dimensions are larger, the Walsh-Hadamard sketch is also faster than ordinary least-squares.  
Overall the results show that the partial compression method can significantly improve on the accuracy in comparison with full compression while significantly reducing the computation time with respect to ordinary least-squares.

			% using the Count sketch,
            % with $A$ a $5\cdot 10^4\times 500$ random matrix with condition number $10^{6}$.
            % (PUT in text?: Default accuracy for Blendenpik set to $10^{3}$ instead of $10^{-16}$.
            % Blendenpik (mex) is the official software package using the spiral WHT; we also 
            % compare to our own matlab implementation of Blendenpik using the count sketch.
\begin{figure}
    \centering
%\begin{minipage}{0.44\linewidth}
%	\centering
	%	\includegraphics[width=0.6\linewidth]{figs/speedTest3_sweepM_v03_resids.pdf}
    	\includegraphics[width=0.55\linewidth]{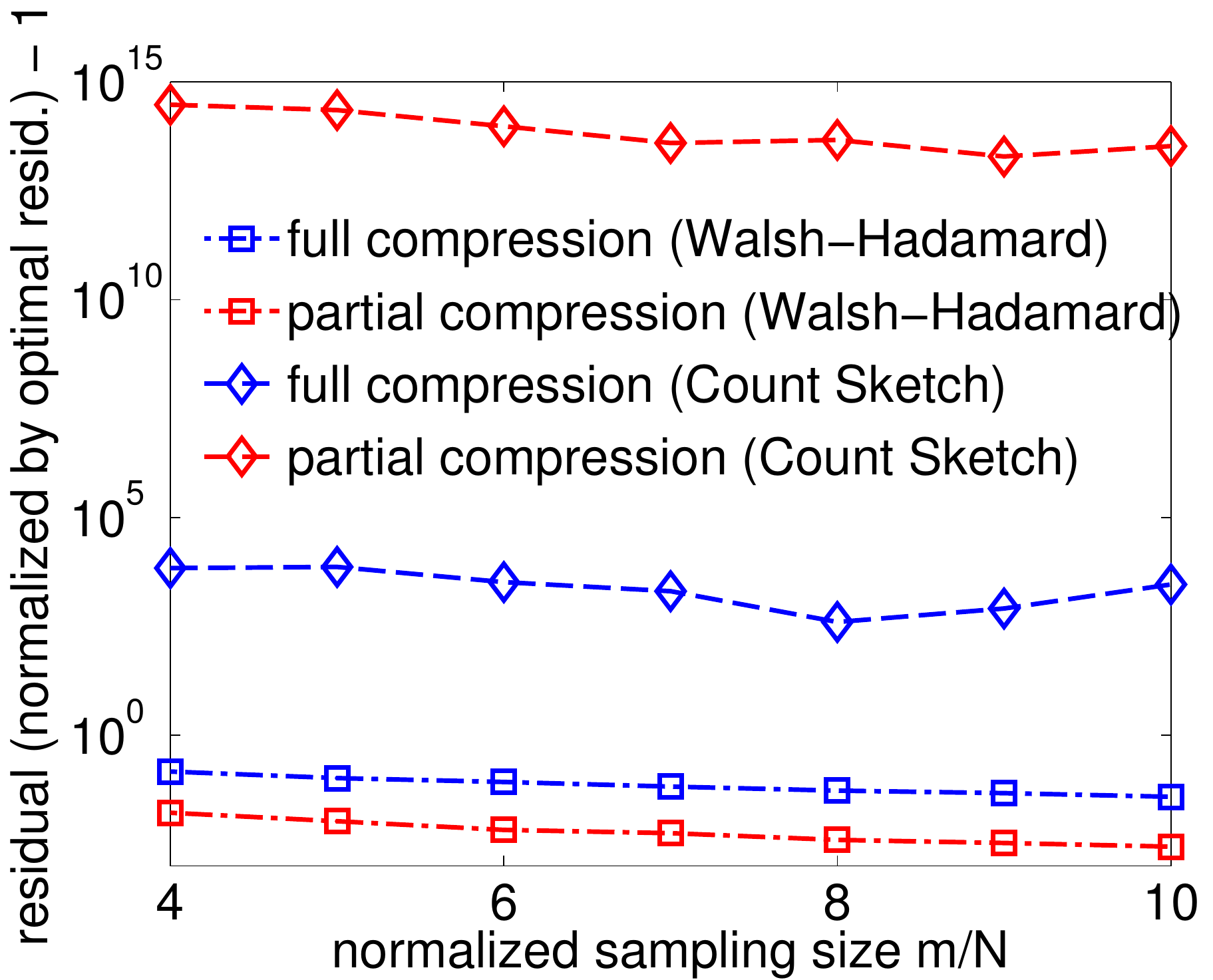}
	%  To see the details:
	%\includegraphics[width=.45\textwidth]{figs/speedTest3_sweepM_v03_resids_annotated.pdf}
	\caption{
		%Speed and accuracy for partial/full compressed least squares for two types of sketching matrices.
		%Left: accuracy.
		%Right: speed results (both partial and full were identical, so only the results of partial compression shown).
	Accuracy of partial/full compressed least squares for ``semi-coherent'' matrices.
		% Also, we only have 3 trials/pt instead of 4, but minor detail
		}		\label{fig:speedTest3_coherent}	
%\end{minipage}
%\hspace{0.05\linewidth}
\end{figure}
\begin{figure}
	\centering
%\begin{minipage}{0.44\linewidth}
%    \centering
		\includegraphics[width=0.55\linewidth]{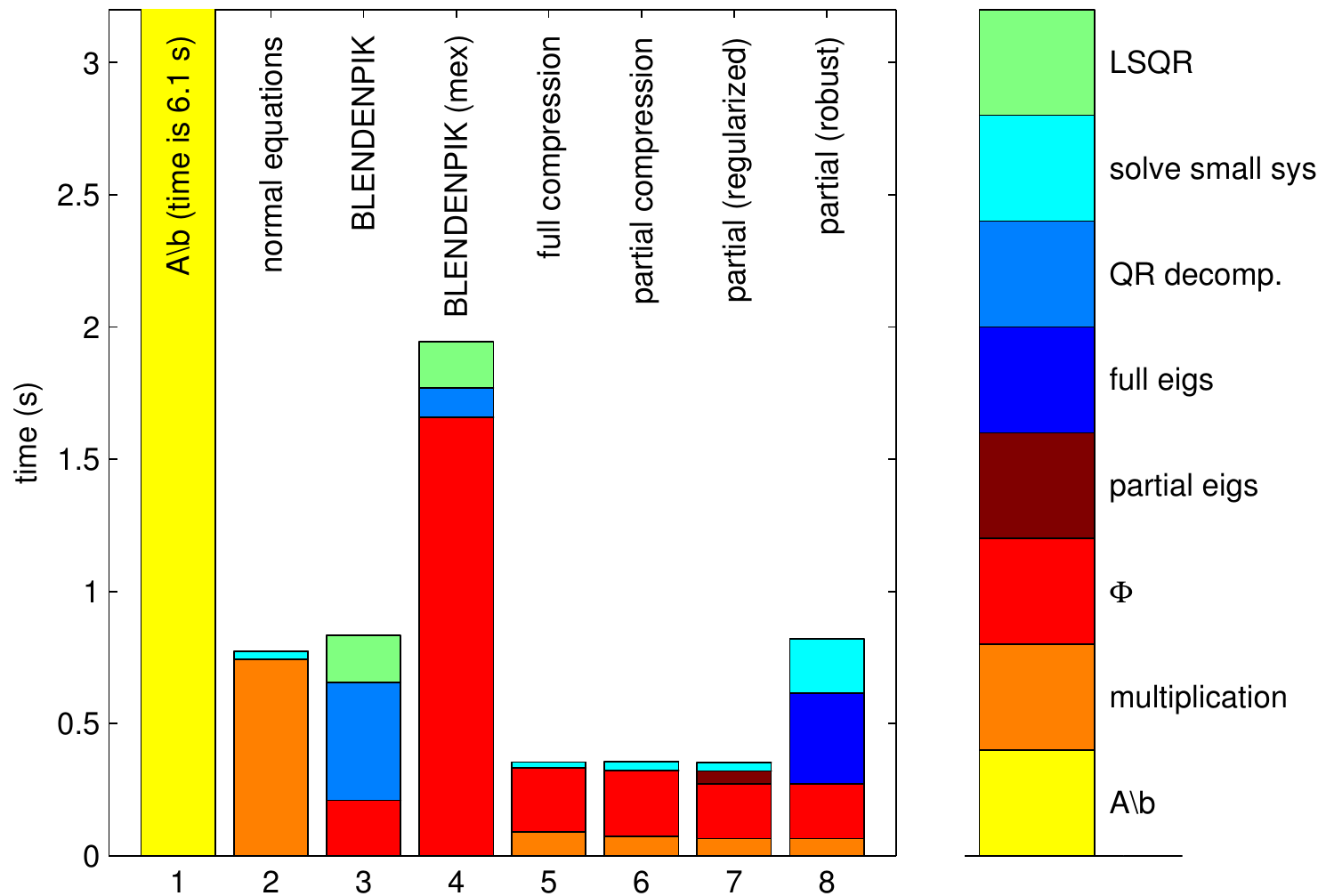}
		\caption{Breakdown of the timing of the parts of the algorithms. Note that the time
for $A \backslash b$ continues off the chart.		
		} \label{fig:breakdown}
%\end{minipage}
\end{figure}

The timing results on the semi-coherent matrix from \cite{avron2010blendenpik} are essentially the same as for the incoherent one, hence we do not show them. However, Figure~\ref{fig:speedTest3_coherent} shows that partial compression combined with the count sketch has a significantly greater error. This is not unexpected since the semi-incoherent matrix is sparse (and in fact parts of it are nearly diagonal) and the count sketch will preserve most of this sparsity, though we lack a theory to explain it precisely. Note that our approximation guarantees were only shown for the Walsh-Hadamard and Gaussian sketch.
%can be expected since both the count sketch and the semi-incoherent matrix are quite sparse. 
The semi-coherence of the matrix does not affect Walsh-Hadamard and Gaussian sketches, in agreement with our theorem.

Figure~\ref{fig:breakdown} shows the breakdown of timing for the individual parts of each of the algorithms that we consider. The compression method used the counting sketch in all compressed methods with the exception of blendenpik (mex) which used the Walsh-Hadamard random matrix vis the sprial WHT package, and both blendenpik versions are set to use low-accuracy for the LSQR step. The matrix $A$ is $5\cdot 10^4\times 500$ random matrix with condition number $10^{6}$.

Finally, Fig.~\ref{fig:smallGamma} investigates the regime when $m$ is very small. The partial and fully compressed solutions greatly deteriorate as $m \rightarrow N$, whereas the robust and regularized versions still give more useful solutions.

\section{Conclusion}

We introduced \emph{partially-compressed} and \emph{robust partially-compressed} least-squares linear regression models. The models reduce the error introduced by random projection, or sketching, while retaining the computational improvements of matrix compression. The robust model specifically captures the uncertainty introduced by the partial compression, unlike ordinary ridge regression with the partially compressed model. Our experimental results indicate that the robust partial compression model out-performs both partially-compressed model (with or without regularization) as well as the fully compressed model. Partial compression alone can also significantly improve the solution quality over full compression.

While the partially-compressed least-squares retains the same computational complexity as full compression, the robust approach introduces an additional difficulty in solving the convex optimization problem. By introducing an algorithm based on one-dimensional parameter search, even the \emph{robust partially-compressed} least-squares can be faster than ordinary least-squares.

% todo: mention also the hybrid idea and that there can be a substantial improvement from computing a linear combination.

%\subsubsection*{Acknowledgments}

\bibliographystyle{plain}
\begin{small}
\bibliography{tempbibfile}
\end{small}

\appendix

\section{Minimization of $h_\tau$ in Algorithm~\ref{alg:superfast}} \label{appx:algorithm}

In this section, we describe how to efficiently minimize $h_\tau(x)$. Recall that $h_\tau$ is defined in \eqref{eq:ht}. % as:
%\[ \h(x) = \dual\left(\|Px\|+\rho\|x\|\right) - b\tp Ax \]
Now consider the problem of computing $\argmin \h(x)$ for a fixed $\dual$. In case 2 of \eqref{eq:cases}, there exists a solution $x \neq 0$ and therefore the function is differentiable and the optimality conditions read
\begin{equation}
    \dual\left( \frac{P\tp P}{\alpha} + \rho\frac{I}{\beta} \right)x = A\tp  b, \quad \alpha=\|Px\|,\; \beta =\|x\|~.
    \label{eq:opt_x}
\end{equation}
The optimality conditions are scale invariant for $x\neq 0$ and therefore we can construct a solution such that $\beta = \|x\| = 1$.

Let $V D V\tp =P\tp  P$ be an eigenvalue decomposition of $P\tp P$, i.e., $D_{ii}=d_i=\sigma_i^2$ are the squared singular values of $P$, and $V$ are the right singular vectors of $P=U\Sigma V\tp $. We make the change-of-variables to $y=V\tp x$ (hence $\|y\|=\|x\|$) and define $\overline{b} = V\tp A\tp  b$, which gives an equation for $y$ which is separable if $\alpha$ is fixed.  We thus need to solve
\begin{gather}
    \dual ( \gamma D + \rho I )y = \overline{b} \label{eq:y1} \\
    1 = \beta = \|y\| \label{eq:y2} \\
    1/\gamma = \alpha= \|\Sigma y\|\label{eq:y3}
\end{gather}
Since $d_i \ge 0$ the solution of \eqref{eq:y1} is unique for a given $\gamma$. Therefore, the equations \eqref{eq:y1}-\eqref{eq:y3} are satisfied if and only if there exists a $\gamma$ such that the solution to \eqref{eq:y1} satisfies both \eqref{eq:y2} and \eqref{eq:y3}.

We use Newton's method to compute $\gamma$ that satisfies \eqref{eq:y1}. Define
\[
    \phi(\gamma) = \dual^{-2}\sum_{i=1}^N  \frac{\overline{b}_i^2}{ \left(\gamma \sigma_i^2 + \rho \right)^2} - 1
\]
so \eqref{eq:y1} and \eqref{eq:y2} are satisfied if $\phi(\gamma)=0$ for $\gamma\ge 0$.
We note that 
\[
    \phi'(\gamma)=  -2\dual^{-2}\sum_{i=1}^N  \frac{\sigma_i^2 \overline{b}_i^2}{ \left(\gamma \sigma_i^2 + \rho \right)^3} 
\]
which is always negative when $\gamma \ge 0$, hence $\phi$ is monotonic and we are guaranteed that there is a unique root (i.e., it is analogous to convex minimization) in the region $\gamma \ge 0$. We can apply any variant of safe-guarded Newton style methods to solve for the root.

Let $\overline{x} = V\tp \overline{y}$ for $\overline{y}$ the optimal solution of the Newton method optimization. We now check if \eqref{eq:y3} is satisfied for this particular value of $\gamma \equiv \alpha^{-1}$ to determine which case of \eqref{eq:cases} we are in. If \eqref{eq:y3} is satisfied and $h_\tau(\overline{x}) = 0$ that means that we are in Case 2 and $\dual = \dual^\star$. That is, the complementary slackness conditions are satisfied and the minimum of $h_\tau(x)$ is 0. If, on the other hand, $h_\tau(\overline{x}) < 0$ then we are in Case 1 and scaling $\overline{x}$ yields an arbitrarily small value. Finally, if $\overline{y}$ does not satisfy \eqref{eq:y3}, then the optimal solution is $y = 0$ and we are in Case 3. Note that $h_\tau(x)$ is not differentiable at $x=0$.

Finally, if we are in Case 2, then $\overline{x}$ is a scaled optimal solution. To recover the optimal solution, we use $t=\tau$ to appropriately scale $\overline{x}$. Specifically, since we took $\beta=1$ and worked with $\gamma\equiv\alpha^{-1}$, this was equivalent to working with $\gamma=\beta/\alpha$ so we can recover the properly scaled $\beta^\star = \alpha^\star \gamma$ and hence $\alpha^\star = (1+\rho\gamma)^{-1}\dual^\star$.

\section{Proof of Theorem~\ref{thm:approx-error-bounds}} \label{app:error_proof}

\begin{proof}
	\newcommand{\xx}{\widehat{x}}
	\newcommand{\xopt}{x\opt}
	\newcommand{\<}{\langle}
	\renewcommand{\>}{\rangle}
The proof uses the stochastic arguments of \cite{Pilanci-Wainwright} directly, and modifies their deterministic argument (Lemma 1).
For brevity, write $\xx=x_\PCLS$ and $\xopt = x_\LS$.	
From the optimality of $\xx$ to the partial-compressed least squares problem \eqref{eq:PCLS}, we have:
\begin{equation} \label{eq:xPCLSopt}
\|\Phi A \xx\|^2 \leq \|\Phi A x  \|^2 + 2 \langle A\,(\xx - x), b\rangle.
\end{equation}
for all $x$, and in particular $x=\xopt$.
%\stephen{Note: we used $x=\xopt$ before, but we can try something else more clever.
%	I tried a simple scalar multiple, but that didn't work.}
%where $\langle\cdot,\cdot\rangle$ denotes the inner product. 
From the optimality of $\xopt$ to equation \eqref{eq:x-LS}, the gradient at $\xopt$ is zero so we have
% NOTE: it should be an equality
$\< A\,x, A \xopt - b \> = 0$
for any $x$, 
and hence, using $x=\xx-\xopt$, 
%\stephen{Note: we used $x=\xx$ before, but we can try something else more clever. I tried a scalar multiple of $\xx$ but it didn't work.}
re-arranging this gives 
\begin{equation}
\label{eq:LSopt}
    \<A(\xx-\xopt),b\> = \< A(\xx-\xopt),A\xopt \>
%	\<A(\xx-\xopt),b\> \le \< A(\xx-\xopt),A\xopt \>
\end{equation}

%Along with
%\begin{align*} 
%\|\Phi A\,(\xx - \xopt)\|^2 &= \|\Phi A \xx\|^2 + \|\Phi A \xopt \|^2 - \\
%&-2 \langle\Phi A \xx, \Phi A \xopt\rangle 
%\end{align*}
%we can show that
Thus
%\begin{gather*}
\begin{align*}
\frac{1}{2} \|\Phi A\,&(\xx - \xopt)\|^2 \\ 
&= \frac{1}{2}\|\Phi A \xx\|^2 + \frac{1}{2}\|\Phi A \xopt \|^2 -\<\Phi A \xx, \Phi A \xopt\> \\
%&\frac{1}{2} \|\Phi A\,(\xx - \xopt)\|^2 \\ %\leq \\ 
&\le \|\Phi A \xopt \|^2 + \langle A\,(\xx - \xopt), b\rangle - \langle\Phi A \xx, \Phi A \xopt\rangle\\
&= \langle A\,(\xx - \xopt), b\rangle - \langle\Phi A \,(\xx - \xopt), \Phi A \xopt\rangle\\
&= \langle A\,(\xx - \xopt), (I-\Phi\tp \Phi)A \xopt \rangle
\end{align*}
%\end{gather*}
where the first inequality follows from \eqref{eq:xPCLSopt} and the final equality follows from \eqref{eq:LSopt}.
%where the last two inequalities follow by adding and subtracting terms and noting that $\langle A\,(\xx - \xopt), (A \xopt - b) \rangle \geq 0$ from the optimality of $\xopt$ \eqref{eq:x-LS}.
%\stephen{here's the issue: we need $\Phi\tp \Phi/m$ not $\Phi\tp\Phi$. The latter doesn't 	makese sense since it doesn't even have the identity as its expectation value. The issue is due to a typo in Pilanci's paper.}
 Normalizing both sides of the last inequality appropriately, we obtain:
\begin{align*}
& \frac{1}{2} \underbrace{\frac{\|\Phi A\,(\xx - \xopt)\|^2}{\|A\,(\xx - \xopt)\|^2}}_{U_1}\|A\,(\xx - \xopt)\|  \\
& \le 
 \|A \xopt \| \underbrace{\left\langle \frac{A\,(\xx - \xopt)}{\|A\,(\xx - \xopt)\|}, (I-\Phi\tp \Phi)\frac{A \xopt}{\|A \xopt \|} \right\rangle }_{U_2} 
\end{align*}
To complete the proof, we need to show that $2\,\frac{U_2}{U_1}$ is bounded above by $\epsilon \in (0,1)$ for both the sub-Gaussian sketch and the ROS sketch. 
Define $Z_1(A) = \inf_{v\in\text{range}(A),\; \|v\|=1}\; \|\Phi v\|^2$
and
\[Z_2(A) = \sup_{v\in\text{range}(A),\; \|v\|=1} \; \left|\left\langle u, \left(\Phi^T\Phi-I\right)v \right\rangle \right|~, \]
where $u$ is any fixed vector of norm $1$.
Then $U_2/U_1 \le Z_2/Z_1$.

Taking the scaling of $\Phi$ into account, then $Z_2/Z_1 < \epsilon$ if: (a) $\Phi$ is a scaled sub-Gaussian sketch and condition \emph{(i)} of the theorem holds, since we apply Lemmas 2 and 3 of \cite{Pilanci-Wainwright}; or (b) $\Phi$ is a scaled ROS sketch and condition \emph{(ii)} of the theorem holds, since we apply Lemmas 4 and 5 of \cite{Pilanci-Wainwright}.
%
%To do so, we rely on results in \cite{Pilanci-Wainwright}; the bound for sub-Gaussian sketches is implied by Lemma~2 and Lemma~3 and for ROS sketches by Lemma~4 and Lemma~5.
%
%\stephen{Note: in particular, here are the details, which is why we need the $1/m$ factor; see equations 41a and 41b in Pilanci}
%Define $Z_1(A) = \inf_{v\in\text{range}(A),\; \|v\|=1}\; \frac{1}{m}\|\Phi v\|^2$
%and
%$Z_2(A) = \sup_{v\in\text{range}(A),\; \|v\|=1} \; \left|\left\langle u, \left(\frac{\Phi^T\Phi}{m}-I\right)v \right\rangle \right| $
%where $u$ is any fixed vector of norm $1$.
%Then $U_2/U_1 \le Z_2/Z_1$, and bounds on $Z_2/Z_1$ follow from Lemmas 2--5 from Pilanci.
\end{proof}
%The proof of Theorem~\ref{thm:approx-error-bounds} follows immediately from Corollary~2 in \cite{Pilanci-Wainwright} and from applying Lemma~\ref{lem:epsilon-optimal-guarantees} above.

\end{document}